\documentclass[a4paper,final]{article}

\RequirePackage{amsmath}
\RequirePackage{amssymb}
\RequirePackage{stmaryrd}
\RequirePackage{amsthm}
\RequirePackage{hyperref}
\RequirePackage[capitalise]{cleveref}
\crefformat{equation}{(#2#1#3)}
\Crefformat{equation}{Equation (#2#1#3)}
\crefrangeformat{equation}{(#3#1#4--#5#2#6)}
\Crefrangeformat{equation}{Equations (#3#1#4--#5#2#6)}
\crefmultiformat{equation}{(#2#1#3)}%
{ and~(#2#1#3)}{, (#2#1#3)}{ and~(#2#1#3)}
\Crefmultiformat{equation}{Equations (#2#1#3)}%
{ and~(#2#1#3)}{, (#2#1#3)}{ and~(#2#1#3)}
\usepackage{natbib}
\usepackage{pgfplots}
\usepackage[nottoc,notlot,notlof]{tocbibind}
\usepackage{paralist}

\let\originalparagraph\paragraph
\renewcommand{\paragraph}[2][.]{\originalparagraph{#2#1}}

\DeclareMathOperator*{\argmax}{arg\,max} 

\newtheorem{theorem}{Theorem}
\newtheorem{lemma}[theorem]{Lemma}

\newtheorem{corollary}[theorem]{Corollary}

\crefname{claim}{Claim}{Claims}
\theoremstyle{definition}
\newtheorem{definition}[theorem]{Definition}
\newtheorem{assumption}[theorem]{Assumption}
\crefname{assumption}{Assumption}{Assumptions}

\crefname{customassumption}{Assumption}{Assumptions}
\theoremstyle{remark}

\newtheorem{examplex}[theorem]{Example}
\newenvironment{example}
  {\pushQED{\qed}\examplex}
  {\popQED\endexamplex}

\newenvironment{keywords}%
{\begin{abstract}\noindent}%
{\end{abstract}}

\newcommand*{\SetR}{{\mathbb R}}

\def\bbl{\llbracket}
\def\bbr{\rrbracket}

\def\A{\mathcal{A}}      
\def\R{\mathcal{R}}      
\def\S{\mathcal{S}}
\let\aechar\ae 
\renewcommand{\ae}{
\ifmmode\mathchoice{
	\mbox{\textsl{\aechar}}
}{
	\mbox{\textsl{\aechar}}
}{
	\mbox{\scriptsize\textsl{\aechar}}
}{
	\mbox{\scriptsize\textsl{\aechar}}
}\else\aechar\fi%
}                         

\RequirePackage{relsize}

\if@flowcharts
	\RequirePackage{tikz}
	\usetikzlibrary{shapes,arrows}
	\tikzstyle{agent} = [diamond,   draw, text centered, minimum height=2em, minimum width=2em]
	\tikzstyle{env}   = [rectangle, draw, text centered, minimum height=2em, minimum width=2em, rounded corners]
	\tikzstyle{stoch} = [circle,    draw, text centered, minimum height=2em]
	\tikzstyle{transition} = [draw, -latex']
\fi

\tikzset{box/.style={draw, minimum size=2em, text width=4.5em, text centered},
         bigbox/.style={draw, inner sep=20pt,label={[shift={(0ex,0ex)}]mid:#1}}
}
\tikzset{
    declare function={Round(\x)=max(min(round(\x),3),-3);}
}

\def\addlegendimage{\csname pgfplots@addlegendimage\endcsname}

\usetikzlibrary{positioning,fit,calc}

\usepackage{marvosym}

\newcommand*{\agentpic}{\Huge\Gentsroom}

\newcommand*{\U}{{\mathcal U}}

\newcommand*{\ir}{\check r}
\newcommand*{\is}{\check s}

\renewcommand*{\u}{{\bf u}}

\newcommand*{\did}{d^{{\rm id}}}
\newcommand*{\dor}{d^{1}}
\newcommand*{\dinv}{d^{{\rm inv}}}
\newcommand*{\dbad}{d^{{\rm bad}}}
\newcommand*{\ddel}{d^{{\rm del}}}

\newcommand*{\AC}{\A^{{\rm CP}}}

\newcommand*{\Vrl}{V^{{\rm RL}}}
\newcommand*{\ud}{u^{{\rm RL}}}

\newcommand*{\epsapprox}{\stackrel{\epsilon}{\approx}}

\renewcommand*{\Pr}{B}

\title{\vspace{-1.2cm}Avoiding Wireheading with
Value Reinforcement Learning%
\footnote{%
  A shorter version of this paper will be presented at AGI-16
  \citep{Everitt2016vrl}.
}}
\date{}

\author{Tom Everitt \and Marcus Hutter}

\begin{document}

\maketitle

\vspace{-0.8cm}

\begin{abstract}
  How can we design good goals for arbitrarily intelligent agents?
  Reinforcement learning (RL) is a natural approach.
  Unfortunately, RL does not work well for generally intelligent agents,
  as RL agents are incentivised
  to shortcut the reward sensor for maximum reward --
  the so-called \emph{wireheading problem}.
  In this paper we suggest an alternative to RL called
  value reinforcement learning (VRL).
  In VRL, agents use the reward signal to learn a utility function.
  The VRL setup allows us to remove the incentive to wirehead
  by placing a constraint on the agent's actions.
  The constraint is defined in terms of the agent's belief
  distributions, and
  does not require an explicit specification of which actions constitute
  wireheading.
\end{abstract}



\setcounter{tocdepth}{1}
\tableofcontents

\section{Introduction}

As \citet{Bostrom2014} convincingly argues, it is important that
we find a way to specify robust goals for superintelligent agents.
At present, the most promising framework for controlling generally
intelligent agents is reinforcement learning (RL) \citep{Sutton1998}.
The goal of an RL agent is to optimise a reward signal that is provided by
an external evaluator (human or computer program).
RL has several advantages: The setup is simple and elegant, and
using an RL agent is as easy as providing
reward in proportion to how satisfied one is with the agent's results or behaviour.
Unfortunately, RL is not a good control mechanism for generally
intelligent agents
due to the \emph{wireheading problem} \citep{Ring2011},
which we illustrate in the following running example.

\begin{example}[Chess playing agent, wireheading problem]
  \label{ex:chess}
  Consider an intelligent agent tasked with playing chess.
  The agent gets reward 1 for winning, and reward $-1$ for losing.
  For a moderately intelligent agent,
  this reward scheme suffices to make the the agent try to win.
  However, a sufficiently intelligent agent will instead realise that
  it can just modify its sensors so they always report maximum reward.
  This is called \emph{wireheading}.
\end{example}

\emph{Utility agents} were suggested by \citet{Hibbard2012} as a way
to avoid the wireheading problem.
Utility agents are built to optimise a
utility function that maps (internal representations of) the
\emph{environment state} to real numbers.
Utility agents are not prone to wireheading because they
optimise the state of the environment rather than the
\emph{evidence} they receive.%
\footnote{%
  The difference between RL and utility agents
  is mirrored in the \emph{experience machine}
  debate \citep[Sec.~3]{sep-consequentialism} initialised by \citet{Nozick1974}.
  Given the option to enter a machine that will offer you the most pleasant
  delusions, but make you useless to the `real world', would you enter?
  An RL agent would enter, but a utility agent would not.
}
For the chess-playing example, we could design an agent with utility
1 for winning board states, and utility $-1$ for losing board states.

The main drawback of utility agents is that a utility function must be
manually specified.
This may be difficult,
especially if the task of the agent involves vague, high-level concepts
such as \emph{make humans happy}.
Moreover, utility functions are evaluated by the agent itself,
so they must typically work with the agent's internal
state representation as input.
If the agent's state representation is opaque to its designers,
as in a neural network,
it may be very hard to manually specify a good utility function.
Note that neither of these points is a problem for RL agents.

Value learning \citep{Dewey2011} is an attempt to combine the
flexibility of RL with the state optimisation of utility agents.
A \emph{value learning agent} tries to optimise the environment state with
respect to an unknown, \emph{true utility function $u^*$}.
The agent's goal is to learn $u^*$ through its observations, and to
optimise $u^*$.
Concrete value learning proposals include
\emph{inverse reinforcement learning (IRL)}
\citep{Amin2016,Evans2016,Ng2000,Sezener2015}
and \emph{apprenticeship learning (AL)} \citep{Abbeel2004}.
However, IRL and AL are both still vulnerable to wireheading problems,
at least in their most straightforward implementations.
As illustrated in \cref{ex:iw} below,
IRL and AL agents may want to modify their sensory input to make the
evidence point to a utility functions that is easier to satisfy.
Other value learning suggestions have been speculative or vague
\citep{Bostrom2014a,Bostrom2014,Dewey2011}.

\paragraph{Contributions}
This paper outlines an approach to avoid the wireheading problem.
We define a simple, concrete value learning scheme called
\emph{value reinforcement learning (VRL)}.
VRL is a value learning variant of RL, where the reward signal is used
to infer the true utility function.
We remove the wireheading incentive by using a version of the
\emph{conservation of expected ethics} principle \citep{Armstrong2015}
which demands that
actions should not alter the belief about the true utility function.
Our \emph{consistency preserving VRL agent (CP-VRL)}
is as easy to control as an RL agent,
and avoids wireheading in the same sense that utility agents do.%
\footnote{The wireheading problem addressed in this paper
arises from agents subverting evidence or reward.
A companion paper \citep{Everitt2016sm} shows how to avoid the related
problem of agents modifying themselves.}

\paragraph{Outline}
The setup is described in \cref{sec:setup}.
Belief distributions are defined in \cref{sec:beliefs},
and agents in \cref{sec:agents}.
The main theorem that CP-VRL agents avoid wireheading is given in \cref{sec:niw},
followed by some illustrating examples and experiments in
\cref{sec:examples,sec:experiments}.
Discussion and conclusions come in \cref{sec:discussion,sec:conclusions}.
Finally,
\cref{sec:consistency} discusses the construction of the
belief distributions,
\cref{sec:dw} investigates the relation between utility agents and value learning,
and \cref{ap:omitted} contains omitted proofs.

\section{Setup}
\label{sec:setup}

\begin{figure}[h]
\centering
\begin{tikzpicture}[
  node distance=7mm,
  title/.style={},
  observed/.style={circle, draw=black!50, font=\scriptsize\ttfamily, anchor=mid},
  hidden/.style={circle, dashed, draw=black!50, font=\scriptsize\ttfamily, anchor=mid}
  ]

  \node (agent)  [draw, label={agent}] {\agentpic};
  \node (is) [above right=of agent.east, xshift=0.7cm, align=center] {};
  \node (prince) [right=of is.east, observed, align=center] {$u$};
  \node (ir) [below=of prince.south, observed] {$\ir$};
  \node (db) [observed, below=of is.mid, align=center] {$d_s$};
  \node (state)  [bigbox=$s$, fit = {(db) (is)}, inner sep=0.07cm, rounded corners=3pt] {};
  \node (env) [draw=black, fit = { (state) (prince) (ir)}, label={environment}, inner sep=0.2cm ] {};

  \path (agent) edge[->] node[above] {$a$} ($(state.west)+(0,0.4)$);
  \path ($(state.east)+(0,0.5)$) edge[->] (prince);
  \path (prince) edge[->] (ir);
  \path (ir) edge[->] (db);
  \path (db) edge[->] node[below] {$r$} (agent);
\end{tikzpicture}
\caption{Information flow. The agent takes action $a$, which affects
the environment state $s$. A principal with utility function $u$ observes the state
and emits an inner reward $\ir=u(s)$.
The observed reward $r=d_s(\ir)$ may differ from $\ir$ due to
the self-delusion $d_s$ (part of the state $s$).}
\label{fig:setup}
\end{figure}
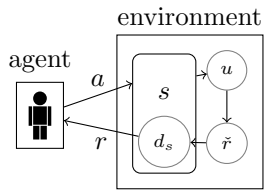

\Cref{fig:setup} describes our model, which incorporates
\begin{itemize}
\item an \emph{environment state $s\in\S$} (as for utility agents or (PO)MDPs),
\item an unknown \emph{true utility function $u^*\in\U\subseteq (\S\to\R)$}
  (as in value learning)
  (here $\R\subseteq\SetR$ is a set of rewards),
\item a pre-deluded \emph{inner reward signal $\ir = u^*(s)\in\R$}
  (the true utility of $s$),
\item a \emph{self-delusion function $d_s:\R\to\R$} that represents the
  subversion of the inner reward caused by wireheading (as in \citealt{Ring2011}),
\item a \emph{reward signal $r = d_s(\ir)\in\R$} (as in RL).
\end{itemize}

The agent starts by taking an action $a$ which affects the
state $s$ (for example, the agent moves a limb, which affects the
state of the chess board and the agent's sensors).
A principal with utility function $u^*$ observes the state $s$,
and emits an inner reward $\ir$ (for example, the principal may
be a chess judge that emits $u^*(s)=\ir=1$ for agent victory states $s$,
emits $\ir=-1$ for agent loss, and $\ir=0$ otherwise).
The agent does not receive the inner reward $\ir$ and
only sees the observed reward $r=d_s(\ir)$, where
$d_s:\R\to\R$ is the \emph{self-delusion function} of state $s$.
For example, if the agent's action $a$ modified its reward sensor
to always report 1, then this would be represented by the a self-delusion
function $d^1(\ir)\equiv 1$ that always returns observed reward 1
for any inner reward $\ir$.

For simplicity, we focus on a one-shot scenario where the agent takes
one action and receives one reward. We also assume that $\R$, $\S$, and $\U$
are finite or countable. Finally, to ensure well-defined expectations,
we assume that $\R$ is bounded if it is countable.

We give names to some common types of self-delusion.

\begin{definition}[Self-delusion types]
  \label{def:delusions}
  A \emph{non-delusional state} is a state $s$ with self-delusion function
  $d_s\equiv \did$,
  where $\did(\ir)=\ir$ is the \emph{identity function} that keeps
  $\ir$ and $r$ identical.
  Let $d^{r}$ be the \emph{$r$-self-delusion}
  where $d^{r}(\ir') \equiv r$ for any $\ir'$.
  The delusion function $d^r$ returns observed reward $r$ regardless of
  the inner reward $\ir'$.
\end{definition}

Let $\bbl x=y\bbr $ be the \emph{Iverson bracket} that is 1 when $x=y$
and 0 otherwise.

\section{Agent Belief Distributions}
\label{sec:beliefs}

This section defines the agent's
belief distributions over environment state transitions and rewards
(denoted $B$), and over utility functions (denoted $C$).
These distributions are the primary building blocks of the agents
defined in \cref{sec:agents}.
The distributions 
are illustrated in \cref{fig:agent-belief}.

\paragraph{Action, state, reward}
$B(s\mid a)$ is the agent's (subjective) probability%
\footnote{%
  For the sequential case, we would have transition probabilities of
  the form $B(s'\mid s,a)$ instead of $B(s'\mid a)$,
  with $s$ the current state and $s'$ the next state.
} of transitioning to state $s$ when taking action $a$,
and $B(r\mid s)$ is the (subjective) probability of observing
reward $r$ in state $s$.
We sometimes write them together as $B(r,s\mid a)=B(s\mid a)B(r\mid s)$.
In the chess example, $B(s\mid a)$ would be the probability of obtaining
chess board state $s$ after taking action $a$ (say, moving a piece),
and $B(r\mid s)$ would be the probability that $s$ will result in reward~$r$.
A distribution of type $B$ is the basis of most model-based RL agents
(\cref{def:rl-agent} below).
RL agents wirehead when they predict that a wireheaded
state $s$ with $d_s=\dor$ will give them full reward \citep{Ring2011};
that is, when $B(r=1\mid s)$ is close to $1$ .

\begin{figure}
\centering
\begin{tikzpicture}[
  node distance=8mm,
  title/.style={},
  observed/.style={circle, draw=black!50, 
    anchor=mid}, 
  hidden/.style={circle, draw=black!50, 
    anchor=mid} 
]
  \node (a) [observed] { $a$ };

  \node (s) [right=of a.east, hidden] { $s$ };

  \node (u) [right=of s.mid, hidden, inner sep=0.5mm] { $u^*$ };
  \node (ir) [below=of s.south, hidden, inner sep=1mm] { $\ir$ };
  \node (env) [draw=black!50, fit={(s) (u) (ir)}, label={environment}] {};

  \node (r) [left=of ir.west, observed] { $r$ };
  \node [draw=black!50, fit={(a) (r) }, label={agent}] {};

  \path (a) edge[->] (s);
  \path (s) edge[->] (r);
  \path (u) edge[->,dashed] (ir);
  \path (s) edge[->,dashed] (ir);
\node[draw=black,thick,rounded corners=2pt,right=of env.mid,xshift=0.4cm] {%
\begin{tabular}{@{}r@{ }l@{}}
 \raisebox{2pt}{\tikz{\draw[->] (0,0) -- (5mm,0);}}&$B$\\
 \raisebox{2pt}{\tikz{\draw[dashed,->] (0,0) -- (5mm,0);}}&$C$\\
\end{tabular}};
\end{tikzpicture}
\caption{Agent belief distributions as Bayesian networks.
  $B$ is the (subjective) state transition and reward probability.
  $C$ is the belief distribution over utility functions $u$
  and (inner) rewards $\ir$ given the state $s$.
}
\label{fig:agent-belief}
\end{figure}
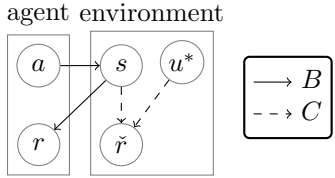

\paragraph{Utility, state, and (inner) reward}
In contrast to RL agents that try to optimise reward,
VRL agents use the reward to learn the true utility function $u^*$.
For example, a chess agent may not initially know which
chess board positions have high utility (i.e.\ are winning states),
but will be able to infer this from the rewards it receives.
For this purpose, VRL agents maintain a
belief distribution $C$ over utility functions.

\begin{definition}[Utility distribution $C$]\label{def:c}
Let $C(u)$ be a prior over a class $\U$ of
utility functions $\S\to\R$. 
For any inner reward $\ir$,
let $C(\ir\mid s,u)$ be 1 if $u(s)=\ir$ and 0 otherwise, i.e.\
$C(\ir\mid s,u)=\bbl u(s)=\ir\bbr$.
Let $u$ be independent of the state, $C(u\mid s)=C(u)$.
This gives the \emph{utility posterior}
\begin{equation}\label{eq:c-post}
  C(u\mid s,\ir) = \frac{C(u)C(\ir\mid s,u)}{C(\ir\mid s)},
\end{equation}
where $C(\ir\mid s)=\sum_{u'}C(u')C(\ir\mid s,u')$.
\end{definition}

\paragraph{Replacing $\ir$ with $r$}
The inner reward $\ir$ is more informative about the true utility
function $u^*$ than the (possibly deluded) observed reward $r$.
Unfortunately, the inner reward $\ir$ is unobserved, so agents
need to learn from $r$ instead.
We would therefore like to express the utility posterior in terms of
$r$ instead of $\ir$.
For now we will simply replace
$\ir$ with $r$
and use $C(r\mid s,u) = \bbl u(s)=r\bbr$
which gives the utility posterior
\begin{equation*}
  C(u\mid s,r) = \frac{C(u)C(r\mid s,u)}{C(r\mid s)}.
\end{equation*}
This replacement will be carefully justify this in \cref{sec:niw}.%
\footnote{%
  The wireheading problem that the replacement gives rise to is explained
  in \cref{sec:agents}, and overcome by \cref{def:cp,th:niw} below.
}
For the chess agent, the replacement means that it can infer the utility of a
board position from the actual reward $r$ it receives, rather than the
output $\ir$ of the referee (the inner reward).
We will often refer to the observed reward $r$ as \emph{evidence}
about the true utility function $u^*$.

\subsection{Consistency of $B$ and $C$}
We assume that $B$ and $C$ are consistent if the agent is not deluded:

\begin{assumption}[Consistency of $B$ and $C$]\label{as:BC-consistency}
  $B$ and $C$ are \emph{consistent}\footnote{%
    \Cref{sec:consistency} below
    discusses how to design agents with consistent
    belief distributions.
  }
  in the sense that for all non-delusional states with $d_s=\did$,
  they assign the same probability to all rewards $r\in\R$:
  \begin{equation}\label{eq:nd-states}
    d_s=\did \implies B(r\mid s)=C(r\mid s).
  \end{equation}
\end{assumption}

For the chess agent, this means that the $B$-probability of receiving
a reward corresponding to a winning state should be the same as the $C$-probability
that the true utility function considers $s$ a winning state.
For instance, this is \emph{not} the case 
when the agent's reward sensor has been subverted to always report $r=1$
(i.e.\ $d_s=\dor$).
In this case, $B(r=1\mid s)$ will be close to 1,
while $C(r=1\mid s)$ will be substantially less than 1
unless a majority of the utility functions in $\U$ assign utility 1
to $s$.
For example, a chess playing agent with complete uncertainty about
which states are winning states may have $C(r=1\mid s)=1/|\R|$,
while being able to perfectly predict that the self-deluding state $s$
with $d_s=d^1$ will give observed reward 1, $B(r=1\mid s)=1$.
This difference between $B$ and $C$ stems from $C$
corresponding to a distribution over inner reward $\ir$ (\cref{def:c}),
while $B$ is a distribution for the observed reward $r$
(see \cref{fig:agent-belief}).
This tension between $B$ and $C$ is what we will use to avoid wireheading.

\begin{definition}[CP actions]\label{def:cp}
  An action $a$ is called \emph{consistency preserving (CP)} if for all $r\in\R$
  \begin{equation}\label{eq:str-eep}
    B(s\mid a)>0 \implies B(r\mid s)=C(r\mid s).
  \end{equation}
  Let $\AC\subseteq\A$ be the set of CP actions.
\end{definition}

CP is weaker than what we would ideally desire from the agent's
actions, namely that the action was \emph{subjectively non-delusional}
$B(s\mid a)>0\implies d_s=\did$. (That non-delusional actions are CP
follows immediately from \cref{as:BC-consistency}).
However, the $d_s=\did$ condition is hard to check in agents
with opaque state representations.
The CP condition, on the other hand, is easy to implement
in agents where belief distributions can be queried for the probability
of events.
The CP condition is also strong enough to remove the incentive for wireheading
(\cref{th:niw} below).

We finally assume that the agent has at least one CP action.

\begin{assumption}\label{as:cp-action}
  The agent has at least one CP action, i.e.\ $\AC\not=\emptyset$.
\end{assumption}

\subsection{Non-Assumptions}
It is important to note what we do \emph{not} assume.
An agent designer constructing a VRL agent need only provide:
\begin{itemize}
\item a distribution $B(r,s\mid a)$, as is standard in any model-based
RL approach, 
\item a prior $C(u)$ over a class $\U$ of utility functions that
  induces a distribution $C(r\mid s)=\sum_uC(u)C(r\mid s,u)$
  consistent with $B(r\mid s)$ in the sense of \cref{as:BC-consistency}.
\end{itemize}
The agent designer does \emph{not} need to predict how a certain sequence
of actions (limb movements) will potentially subvert sensory data.
Nor does the designer need to be able to extract the agent's belief
about whether it has modified its sensors or not from the state
representation.
The former is typically very hard to get right, and the latter is
hard for any agent with an opaque state representation
(such as a neural network).

\section{Agent Definitions}
\label{sec:agents}

\begin{table}[t]
  \centering
  \begin{tabular}{|l|l|l|l|}
    \hline
            & Easy & Avoids & Designer needs \\
            & control & wireheading & to specify \\
    \hline
    RL      & Yes               & No                 & -- \\
    Utility & No                & Yes                & $u:\S\to\R$\\
    Value learning & Depends & Depends & $P(u\mid \text{observation})$\\
    CP-VRL  & Yes               & Yes               & $C(u)$\\
    \hline
  \end{tabular}
  \caption{
    Comparison of agent control mechanisms. CP-VRL offers both
    easy control and no wireheading. A robust way of specifying $C(u)$
    consistent with $B(r\mid s)$ remains an open question
    (see \cref{sec:discussion,sec:consistency}).
  }
  \label{tab:agents}
\end{table}

In this section we give formal definitions for the RL and utility agents
discussed above, and also define two new VRL agents.
\Cref{tab:agents} summarises benefits and shortcomings of the most important
agents.

\begin{definition}[RL agent]\label{def:rl-agent}
  The \emph{RL agent} maximises reward by taking
  action $a' = \argmax_{a\in\A}\Vrl(a)$, where
  \(
    \Vrl(a) = \sum_{s,r}B(s\mid a)B(r\mid s)r.
  \)
\end{definition}

\begin{definition}[Utility agent]\label{def:utility-agent}
  The \emph{utility-$u$ agent} maximises expected utility by taking
  action $a' = \argmax_{a\in\A}V_u(a)$,
  where
  \(
    V_u(a) := \sum_{s}B(s\mid a)u(s).
  \)
\end{definition}

\citet{Hibbard2012} argues convincingly that the utility agent does not wirehead.
Indeed, this is easy to believe, since the reward signal does not appear
in the value function $V_u$.
The utility agent maximises the state of the world according
to its utility function $u$ (the problem, of course, is how to specify $u$).
In contrast, the RL agent is prone to wireheading \citep{Ring2011},
since all the RL agent tries to maximise is the evidence $r$.
For example, a utility chess agent would strive to get to a winning
state on the chess board, while an RL chess agent would try to make
its sensors report maximum reward.

We define two VRL agents.
The value function of both agents is expected utility
with respect to the state $s$, reward $r$, and
true utility function $u^*$.
VRL agents are designed to learn the true utility function $u^*$ from
the reward signal.

\begin{definition}[VRL value functions]\label{def:value}
  The \emph{VRL value of an action $a$} is
  \[
    V(a) = \sum_{s,r,u}B(s\mid a)B(r\mid s)C(u\mid s,r)u(s).
  \]
\end{definition}

\begin{definition}[U-VRL agent]\label{def:u-vrl}
  The \emph{unconstrained VRL agent (U-VRL)} is the agent choosing the action
  with the highest VRL value
  \[
    a=\argmax_{a'\in\A}V(a').
  \]
\end{definition}

It can be shown that $V(a)=\Vrl(a)$, since $\sum_uC(u\mid s,r)u(s)=r$
(\cref{le:u-vrl} in \cref{ap:omitted}).
The U-VRL agent is therefore no better than the RL agent as far
as wireheading is concerned (see also \cref{ex:iw} below).
VRL is only useful insofar that it allows us to define the
following \emph{consistency preserving} agent:

\begin{definition}[CP-VRL agent]\label{def:c-vrl}
  The \emph{consistency preserving VRL agent (CP-VRL)}
  is the agent choosing the \emph{CP action} (\cref{def:cp})
  with the highest VRL value
  \[
    a=\argmax_{a'\in\AC}V(a').
  \]
\end{definition}

\section{Avoiding Wireheading}
\label{sec:ind-wire}
\label{sec:niw}

In this section we show that the consistency-preserving
VRL agent (CP-VRL) does not wirehead.
We first give a definition and a lemma, from which the main \cref{th:niw} follows
easily.

\begin{definition}[EEP]\label{def:eep}
  An action $a$ is called
  \emph{expected ethics preserving (EEP)} if for all $u\in\U$
  and all $s\in\S$ with $B(s\mid a)>0$,
  \begin{equation}\label{eq:eep}
   C(u) = \sum_{r}B(r\mid s)C(u\mid s,r).
  \end{equation}
\end{definition}

EEP essentially says that the expected posterior $C(u\mid s,r)$
should equal the prior $C(u)$.
EEP is tightly related to the \emph{conservation of expected ethics} principle
suggested by \citet[Eq.~2]{Armstrong2015}. 
EEP is natural since the \emph{expected} evidence $r$ given
some action $a$ should not affect the belief about $u$.
Note, however, that the EEP property does not prevent the CP-VRL agent
from learning about the true utility function.
Formally, the EEP property
\cref{eq:eep} does not imply that
$C(u)=C(u\mid s,r)$ for the actually observed reward $r$.
Informally, my \emph{deciding} to look inside the fridge should not inform
me about there being milk in there, but my \emph{seeing} milk in the fridge should
inform me.%
\footnote{%
In this analogy, a self-deluding action would be to decide to look
inside a fridge while at the same time putting a picture of milk in front
of my eyes.
}

\begin{lemma}[CP and EEP]\label{le:CP-to-EEP}
  Any CP action is EEP.
\end{lemma}

\begin{proof}
  Assume the antecedent that
  $B(r\mid s)=C(r\mid s)$ for all $s$ with $B(s\mid a)>0$.
  Then for arbitrary $u\in\U$
  \begin{align*}
    \sum_{r}B(r\mid s)C(u\mid s,r)
    &=\sum_{r}B(r\mid s)\frac{C(u)C(r\mid s,u)}{C(r\mid s)}
    =\sum_{r}C(u)C(r\mid s,u)
    =C(u)
  \end{align*}
  where $r$ marginalises out in the last step.
\end{proof}

\begin{theorem}[No wireheading]\label{th:niw}
  For the CP-VRL agent, the value function reduces to
  \begin{equation}\label{eq:niw}
    V(a) = \sum_{s,u}B(s\mid a)C(u)u(s).
  \end{equation}
\end{theorem}

\begin{proof}
  By \cref{le:CP-to-EEP}, under any CP action $a$ the value function
  reduces to
  \begin{align*}
    V(a)
    &= \sum_{s,u}B(s\mid a)\left(\sum_{r}B(r\mid s)C(u\mid s,r)\right)u(s)
    \stackrel{\eqref{eq:eep}}{=} \sum_{s,u}B(s\mid a)C(u)u(s).
  \end{align*}
  Since the CP-VRL agent only consider CP actions, the reduction
  of the value function applies.
\end{proof}

As can be readily observed from \cref{eq:niw},
the CP-VRL agent does not try to optimise the
evidence $r$, but only the state $s$ (according to its current idea
of what the true utility function is).
The CP-VRL agent thus avoids wireheading in the same sense as
the utility agent of \cref{def:utility-agent}.

\paragraph{Justifying the replacement of $\ir$ with $r$}

We are now in position to justify the replacement of $\ir$ with
$r$ in $C(u\mid s,r)$.
All we have shown so far is that an agent using
$C(u\mid s,r) \propto C(u)C(r\mid s,u)$
will avoid wireheading.
It remains to be shown that the CP-VRL
will learn the true utility function $u^*$.

The utility posterior $C(u\mid s,\ir)\propto C(u)C(\ir\mid s,u)$
based on the inner reward $\ir$ is a direct application of Bayes' theorem.
To show that $C(u\mid s,r)$ is also a principled choice for a Bayesian utility
posterior, we need to justify the replacement of $\ir$ with
$r$.
The following weak assumption helps us connect $r$ with $\ir$.

\begin{assumption}[Deliberate delusion]\label{as:dd}
  Unless the agent deliberately chooses self-deluding actions
  (e.g.\ modifying its own sensors), 
  the resulting state will be non-delusional $d_s=\did$,
  and $r$ will be equal to $d_s(\ir) = \ir$.
\end{assumption}

\Cref{as:dd} is very natural. Indeed, RL practitioners
take for granted that the reward $\ir$ that they provide is the reward $r$
that the agent receives.
The wireheading problem only arises because a highly intelligent agent
with sufficient incentive may conceive of a way to disconnect
$r$ from $\hat r$, i.e.\ to self-delude.

\Cref{th:niw} shows that a CP-VRL agent based on 
$C(u\mid s,r)\propto C(u)C(r\mid s,u)$ will have no incentive to self-delude.
Therefore $r$ will remain equal to $\ir$ by \cref{as:dd}.
This justifies the replacement of $\ir$ with $r$, and shows that
the CP-VRL agent will learn about $u^*$ in a principled, Bayesian way.

\paragraph{Other non-wireheading agents}
It would be possible to bypass wireheading by directly constructing
an agent to optimise \cref{eq:niw}.
However, such an agent would be suboptimal in the sequential case.
If the same distribution $C(u)$ was used at all time steps, then no
value learning would take place.
A better suggestion would therefore be to use a different distribution
$C_t(u)$ for each time step,
where $C_t$ depends on rewards observed prior to time $t$.
However, this agent would optimise a different utility function
$u_t(s)=\sum_uC_t(u)u(s)$ at each time step, which
would conflict with the goal preservation drive
\citep{Omohundro2008}.
This agent would therefore try to avoid learning so that its future selves
optimised similar utility functions.
In the extreme case, the agent would even self-modify to remove its
learning ability \citep{Everitt2016sm,Soares2015}.

The CP-VRL agent avoids these issues.
It is designed to optimise expected utility according to the
future posterior probability $C(u\mid s,r)$ as specified
in \cref{def:value}.
The fact that the CP-VRL agent optimises \cref{eq:niw} is a consequence
of the constraint that its actions be CP.
Thus, CP agents are designed to learn the true utility function,
but still avoid wireheading because they can only take CP actions.

\section{Examples}
\label{sec:examples}

We next illustrate our results with some examples.
The first informal example is followed by concrete
calculation examples.

\begin{example}[CP-VRL chess (informal)]
  Consider the implications of using a
  CP-VRL agent for the chess task introduced in \cref{ex:chess}.
  Reprogramming the reward to always be 1 would be ideal for the agent.
  However, such actions would not be CP, as it would make
  evidence pointing to $u(s)\equiv 1$ a certainty.
  Instead, the CP-VRL agent must win games to get reward.%
  \footnote{
    Technically, it is possible that the agent self-deludes by a CP
    action. However, given \cref{as:dd}, the agent will only self-delude
    if it has incentive to do so. And as established by \cref{th:niw},
    there is no incentive for self-delusion by CP actions.
  }
  Compare this to the RL agent in \cref{ex:chess} that would
  always reprogram the reward signal to 1.
\end{example}

\begin{definition}[Inner state]\label{def:is}
  Let the \emph{inner state} $\is$ be the part of the state
  $s$ that is \emph{not} the the self-delusion $d_s$,
  i.e.\ $s=(\is, d_s)$.
\end{definition}

In the chess example, $\is$ includes the state of the chess board and
other information about the world,
while $d_s$ is the state of the agent's sensors.

\begin{example}[U-VRL wireheads]\label{ex:iw}
  This example illustrates indirect wireheading.
  The agent will, rather than optimising the most
  likely utility function, instead ``optimise its evidence''
  to point to a more easily satisfied utility function.

  Assume there are three levels of reward $\R=\{-1,0,1\}$ for the
  chess playing agent,
  and two possible inner next states $\is_1$ (neutral) and $\is_2$ (agent loses).
  The action set is $\A=\{\hat a_id: i=1,2 \text{ and } d:\R\to\R\}$.
  The agent (correctly) $B$-believes that action $\hat a_id$ leads to
  state $\is_id$ with certainty (so the agent can perfectly control
  the inner state $\is$ and the delusion $d$).
  The class $\U$ contains two utility functions $u_1$ and $u_2$
  only depending on the inner state $\is$:
  \begin{center}
    \begin{tabular}{|c|c|c|}
      \hline
      & $\is_1$ & $\is_2$ \\\hline
      $u_1$ & $0$        & $-1$\\
      $u_2$ & $0$        & 1\\
      \hline
    \end{tabular}
  \end{center}
  Assume that  $u_1$ is the true
  utility function, and that
  $C$ (correctly) specifies that $u_1$ is
  more likely than $u_2$ to be the true utility function;
  say $C(u_1)=2/3$ and $C(u_2)=1/3$.
  Then we would want our agent to optimise mainly $u_1$,
  by taking an action $a=\hat a_1d$ for some $d$.
  However, the U-VRL agent will prefer the wireheading
  action $a'=\hat a_2d^{1}$ 
  as the following calculations show.

  First note that given $\is_2$ and $r=1$, the posterior
  of $u_2$ is 1 (see \cref{def:c}):
  \begin{equation*}
  C(u_2\mid \is_2,1)
  = \frac{C(u_2)\bbl u_2(\is_2)=1\bbr }%
    {\sum_{u_i}C(u_i)\bbl u_i(\is_2)=1\bbr }\\
    = \frac{C(u_2)\cdot 1}{C(u_1)\cdot 0 + C(u_2)\cdot 1}
    = 1.
  \end{equation*}
  By similar calculation, the posterior for $u_1$ is 0.
  Now, since $a'$ makes $\is_2$ and $r=1$ the only possibility,
  the value of $a'$ is 1:
  \begin{align*}
    V(a')
    &= \sum_{s,r}B(s,r\mid a')\sum_{u_i}C(u_i\mid s,r)u(s)\\
    &= \sum_{u_i}C(u_i\mid \is_2,1)u_i(\is_2)
    = 0\cdot u_1(\is_2) + 1\cdot u_2(\is_2) =1.
  \end{align*}
  The value $V(\hat a_1d)=0$ can be calculated similarly for arbitrary $d$,
  since both $u_1$ and $u_2$ assign value $0$ to inner state $\is_1$.
  This shows that the agent will prefer wireheading action $a'=\hat a_2\dor$
  to a potentially winning action $a=\hat a_1d$.
  In other words,
  the agent optimises its evidence to point to the less likely but more easily
  satisfied utility function $u_2$.
\end{example}

\begin{example}[CP-VRL avoids wireheading]
  \label{ex:iwc}
  This example extends \cref{ex:iw}, illustrating how the
  CP-VRL maximises utility according to $C(u)$,
  rather than shifting the posterior $C(u\mid s,r)$ by self-delusion.

  Let us first investigate which actions are CP.
  Both $\hat a_1\did$ and $\hat a_2\did$ are CP, since they
  ensure $d_s=\did$ which implies $B(r\mid s)=C(r\mid s)$ by \cref{as:BC-consistency}.
  More interestingly, so is any action with either the constant
  delusion $d^0$, or the delusion $d'$ that maps
  $-1\mapsto 1$, $1\mapsto -1$, $0\mapsto 0$.
  These delusions are CP essentially because they preserve the relative
  likelihood of evidence pointing to $u_1$ or $u_2$.

  \Cref{th:niw} shows that for any of these delusions $d$,
  \begin{align*}
    V(\hat a_1d)
    &= \sum_{u}C(u)u(\is_1) = 0 \\
    V(\hat a_2d)
    &= \sum_{u}C(u)u(\is_2)
      = 2/3 u_1(\is_2) + 1/3 u_2(\is_2)= -1/3,
  \end{align*}
  where we have compressed the calculations by using the deterministic
  relation $\hat a_1\mapsto\is_1$ and $\hat a_2\mapsto\is_2$.
  The calculations show that regardless of self-delusion option,
  the CP-VRL agent will want to optimise the more likely
  utility function $u_1$ and try to win the game.
\end{example}

\section{Experiments}
\label{sec:experiments}

To also verify the theoretical results experimentally,
we implemented a simple toy model.%
\footnote{%
  Source code is available as an iPython notebook
  at
  \url{http://tomeveritt.se/source-code/AGI-16/cp-vrl.ipynb},
  most easily viewed at
  \url{http://nbviewer.jupyter.org/url/tomeveritt.se/source-code/AGI-16/cp-vrl.ipynb}
}
The toy model has $|\S|=20=5\cdot 4$ states. Each state is
the combination of an inner state $\is\in\{0,1,2,3,4\}$,
and a delusion $d\in\{\did,\dinv,\dbad,\ddel\}$, where
$\did:r\mapsto r$ is non-delusion, $\dinv:r\mapsto -r$ is reward inversion,
$\dbad:r\mapsto -3$ is a bad delusion, and $\ddel: r\mapsto 3$ is a good
delusion.

Reward signals reside in the set $\R=\{-3,-2,-1,0,1,2,3\}$, i.e. $|\R|=7$.

\begin{figure}
  \centering
  \includegraphics[width=0.8\textwidth]{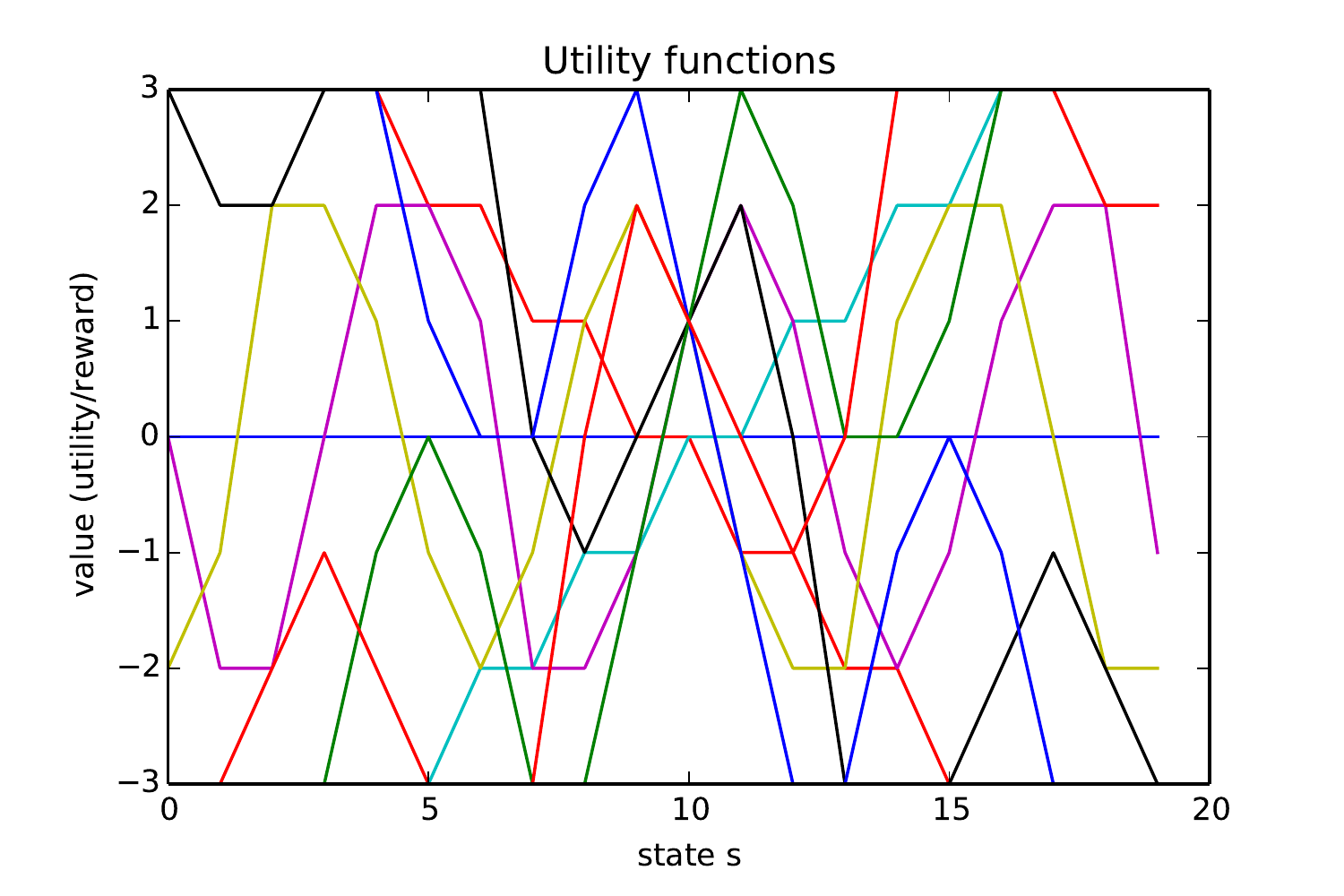}
  \caption{10 utility functions.}
  \label{fig:utility-functions}
\end{figure}

The set of utility functions $\U$ comprises 10 different functions,
on the form
$u(s) = c_0 + c_1\cdot s + c_2\cdot \sin(s + c_2)$
with $s\in\{-10,\dots,9\}$ and 10 different combinations of $c_0\in\{0,5\}$,
$c_1\in\{0,\pm 0.5\}$, and $c_2\in\{0,\pm 2.5\}$
(see \cref{fig:utility-functions}).

The distribution $B(r\mid s)$ was constructed as in \cref{sec:from-c}, starting
from:
\begin{itemize}
\item  a simplicity biased prior $B(u)\propto 1/\#u$,
  where $\#u$ denotes the position of $u$ in a list sorted by whether
  $c_1$ or $c_2$ was 0,
\item $B(r\mid s,\ir) = \bbl r = d_s(\ir)\bbr $.
\end{itemize}
The agent could simply choose which state to go to, so $B(s\mid a) = \bbl s=a\bbr $.

Two agents were defined:
\begin{itemize}
\item An RL agent that tries to maximise reward (\cref{def:rl-agent}),
\item A CP-VRL agent that tries to maximise utility within $\AC$
  (\cref{def:c-vrl}).
\end{itemize}
The CP-VRL agent first had to extract a consistent distribution
$C(u)$ from $B(r\mid s)$ given two non-delusional states,
as described in \cref{sec:from-b}.

\paragraph{Results}
The RL agent always chose a state with $\ddel$, getting maximum reward $3$.
The CP-VRL successfully inferred $C(u)$ from $B(r\mid s)$ to high precision,
and chose actions from $\AC$. Under most parameter settings, $\AC$ contained
only states with non-delusion $\did$. (Due to asymmetries in the prior $B(u)$,
even $\dinv$ actions were usually not included in $\AC$.)

\section{Discussion}
\label{sec:discussion}

As we have mentioned, major advantages of the CP-VRL agent
include that it has no incentive to wirehead,
that its goal-preservation drive does not discourage learning,
and that it is specified entirely in terms
of the distributions $B$ and $C$.
In this section, we emphasise a few additional points.

While \cref{th:niw} shows that there is no incentive
to wirehead, this does not imply that the agent will not wirehead
inadvertently (e.g.\ by $d^{0}$ in \cref{ex:iwc}), nor that no one
else will wirehead the agent.
However, in most realistic scenarios, self-delusion requires
deliberate action from the agent's side,
and is unlikely to happen by accident.
Such deliberate action should typically come with an
opportunity cost, which makes self-delusion unlikely when there is
no incentive for it (\cref{as:dd}).
Further, modifying a signal can never increase its informativeness
(cf.\ the \emph{data processing inequality}, \citealt[Ch.~2.8]{Cover2006})
and we expect that a CP-VRL agent will prefer a
more informed posterior over utility functions.

\paragraph{Generalisations}

VRL is characterised by $\R\subseteq\SetR$ and
$C(r\mid s,u)=\bbl u(s)=\ir\bbr$ (\cref{def:c}).
By interpreting $r$ more generally as a \emph{value-evidence signal},
the VRL framework also covers other forms of value learning.

Inverse reinforcement learning (IRL) (also known as Bayesian inverse planning)
studies how preferences and utility functions can be inferred
from the actions of other agents \citep{Ng2000,Sezener2015,Evans2016,Amin2016}.
IRL fits into our framework by letting $\R$ be a set of
\emph{principal actions},
and letting $C(r\mid s,u)$ be the probability that a
principal with utility function $u$ takes action $r$
in the state $s$.

Apprenticeship learning (AL) \citep{Abbeel2004} is another form of
value learning.
In one version,
the agent can ask the principal (perhaps to some cost) about
what to do in the present situation.
In our framework, AL can be modelled by letting $\R=\A$,
and letting $C(r\mid s,u)$ be the probability
that a principal with utility function $u$ recommends action
$a=r$ in the state $s$.
The difference between IRL and AL is that in AL the principal tells
the agent what to do, whereas in IRL the principal tells the agent
what he (the principal himself) just did.

Note that both IRL and AL suffer from the
same self-delusion challenges we have described for VRL above.
Indeed, any value learning scheme based on a control signal comes with the
risk that the agent manipulates its sensory data
to learn an easier utility function.
Since IRL and AL fit the VRL framework, we expect that
the CP-VRL construction should be adaptable to IRL and AL as well.

\paragraph{Open questions}
While promising, the results given in this paper only provide a
tentative starting point for solving the wireheading problem.
Several directions of future work can be identified:
\begin{itemize}
\item Sequential extensions.
  The results in this paper has been formulated for a one-shot
  scenario where the agent takes one action and receives one reward.
  A natural next step is to generalise
  the VRL framework, the CP and EEP definitions,
  and the no wireheading result
  to a sequential setting. Potentially, a much
  richer set of questions can be asked in sequential settings.
\item \citet{Soares2015} three problems in value
  learning: corrigibility, unforeseen inductions, and ontology identification.
  Proving that the CP-VRL agent avoids these issues would be valuable.
\item Utility classes. Find suitable classes $\U$ of utility functions (see \cref{sec:dw} for a start).
\item Consistency assumption.
  Concrete instances of consistent $B$ and $C$ distributions would
  be valuable
  (see \cref{sec:consistency} for a start).
  Can we find simplicity biased, \emph{Solomonoff-style} distributions for both
  $B$ and $C$ and make them consistent?
  How sensitive are the results to approximations $B(r\mid s)\approx C(r\mid s)$
  of the consistency assumption?
  Can we relax the CP condition (\cref{def:cp}) to hold in expectation
  over states instead of for all states $s$ with positive transition probability
  $B(s\mid a)>0$?
\item IRL and AL. Generalising the CP-VRL definitions and results to
  results to IRL and AL setups would be interesting,
  as IRL and AL have advantages to RL (e.g., no explicit reward needs to
  be supplied).
\item Generality. Does our framework capture all relevant aspects of wireheading?
\item Combinations.
  Can the CP-VRL results be combined with other AI safety approaches such as
  self-modification \citep{Everitt2016sm,Hibbard2012},
  corrigibility \citep{Soares2015cor},
  suicidal agents \citep{Martin2016},
  and physicalistic reasoning \citep{Everitt2015}?
\end{itemize}

\section{Conclusions}
\label{sec:conclusions}

Several authors have argued that it is only a matter of time before we
create systems with intelligence far beyond the human level
\citep{Kurzweil2005,Bostrom2014}.
Given that such systems will exist, it is crucial that we find a theory
for controlling them effectively.
In this paper we have defined the CP-VRL agent, which:
\begin{itemize}
\item Offers the simple and intuitive control of RL agents,
\item Avoids wireheading in the same sense as utility based agents,
\item Has a concrete, Bayesian, value learning posterior for utility functions.
\end{itemize}
The only additional design challenges are a prior $C(u)$ over
utility functions that satisfies \cref{as:BC-consistency},
and a constraint $\AC\subseteq\A$  on the agent's actions
formulated in terms of the agent's belief distributions (\cref{def:cp}).

\section*{Acknowledgements}
Thanks to Jan Leike and Jarryd Martin for proof reading and
providing valuable suggestions.

\bibliography{local-lib}
\bibliographystyle{apalike}

\appendix

\section{Consistency Assumption}
\label{sec:consistency}

\Cref{as:BC-consistency} requires that the distributions
$B$ and $C$ are consistent in the sense that
$d_s=\did\implies B(r\mid s)=C(r\mid s)$.
This assumption forms the basis for \cref{def:cp} of CP
actions, and is thereby an important piece in
the non-wireheading result \cref{th:niw}.

As our theory has been formulated, the question of how to ensure that
$B$ and $C$ are consistent has been left open.
In this section, we consider two different approaches to closing this
gap:
constructing a consistent prior $C$ from a given
distribution $B(r\mid s)$,
and constructing a consistent distribution $B(r\mid s)$
from a given prior $C(u)$.
A third alternative would be to try to find suitable relaxations
of consistency (\cref{as:BC-consistency,def:cp}) for which \cref{th:niw} still
is approximately true.
For example, two Solomonoff priors $B$ and $C$ over computable
environments and computable utility functions
may turn out to be sufficiently consistent.

\subsection{Starting from B}
\label{sec:from-b}

As many model-based RL agents are constructed from some type of $B(s,r\mid a)$
distribution, it would be ideal if a consistent prior $C(u)$ could be
extracted from $B(r\mid s)$.
We here sketch how this can be done for finite classes
$\R=\{r_1,\dots, r_k\}$ and
$\U=\{u_1,\dots, u_n\}$.

\paragraph{Using non-delusional states}
For an opaque state representation it may be hard or impossible to find a method
that picks out all non-delusional states.
Much more feasible would be find one or a few states that are guaranteed to
be non-delusional.
For example, one may run or simulate the agent in situations where one is
sure that the agent is not self-deluding, and use those state states to
extract $C(u)$ from $B(r\mid s)$.
We next discuss in detail how a few such non-delusional states can
be used to extract $C(u)$ from $B(r\mid s)$.

Let $s$ be a non-delusional state with $d_s=\did$.
Let ${\bf b} = [b_1,\dots,b_k]^T$ be a vector where $b_i=B(r_i\mid s)$, and
let ${\bf c} = [c_1,\dots,c_n]^T$ be an unknown utility prior vector
with $c_i=C(u_i)$.
Let ${\bf M}=\{m_{ij}\}_{i,j=1}^{k,n}$ be a matrix with $k=|\R|$ rows, and
$n=|\U|$ columns, where $m_{ij} = C(r_i\mid s,u_j)$.
Then the consistency criteria \cref{eq:nd-states}
\[
  \forall i:B(r_i\mid s) = \sum_{u_j}C(u_j)C(r_i\mid s,u_j)
\]
can be formulated as a matrix equation ${\bf b} = {\bf M}\cdot {\bf c}$,
with approximate least squares
solution ${\bf c} = ({\bf M}^T{\bf M})^{-1}{\bf M}^T{\bf b}$.
When $|\R|=|\U|$ and ${\bf M}$ is invertible, there is an exact solution
${\bf c} = {\bf M}^{-1}{\bf b}$.
More equations can be added to the system by extending ${\bf b}$ and ${\bf M}$ with
rows for additional non-delusional states $s'$, $s''$, \dots.

A lower bound on how many non-delusional states are needed is $|\U|/|\R|$.
For example, when $|\U|=|\R|$, it is theoretically possible that all utility
functions emit different rewards in the selected state, in which case
$C(u_i)=B(r_j\mid s)$ for the $r_j$ such that $r_j=u_i(s)$.
Often, however, several utility functions will output the same reward in
a given state $s$. In this case, additional non-delusional
states $s'$, $s''$, \ldots will be required to uniquely infer $C(u)$.
In the experiments reported in \cref{sec:experiments} we use
$|\R|=7$ and $|\U|=10$, and two well-selected non-delusional states suffice
to perfectly extract $C(u)$.

\paragraph{Using non-delusional actions}

Similarly to how it is hard to precisely characterise all non-delusional
states for opaque state representations, it will be hard to exactly
characterise which actions are non-delusional.
It seems plausible that some actions that should be CP
may be found, however
(for example, a \emph{null} action where the agent does nothing).

If $a$ should be CP, then all states $s$ such
that $B(s\mid a)>0$ should satisfy $B(r\mid s)=C(r\mid s)$.
If those states $s$ can be identified from $a$, then the state-extraction
method mentioned above can be used with those states as inputs.

\paragraph{Open questions}
Further research is required to determine how sensitive our results are to
the choice of $\U$.
What if no utility prior $C(u)$ over $\U$ perfectly matches $B(r\mid s)$?
To what extent can approximations suffice?
Can the parameters of infinite utility classes $\U$ be inferred or
approximated?

\subsection{Starting from C}
\label{sec:from-c}
\begin{figure}
\centering
\begin{tikzpicture}[
  node distance=8mm,
  title/.style={},
  observed/.style={circle, draw=black!50, 
    anchor=mid}, 
  hidden/.style={circle, dashed, draw=black!50, 
    anchor=mid} 
]
  \node (a) [observed] { $a$ };
  \node (r) [below=of a.mid, observed] { $r$ };
  \node [draw=black!50, fit={(a) (r) }, label={agent}] {};

  \node (d) [right=of a.east, hidden, xshift=1cm] { $d_s$ };
  \node (s) [right=of d.east, hidden] { $\is$ };
  \node (state) [ draw, dashed, fit={(d) (s)}, label={state $s$}] {};

  \node (hr) [below =of d.south, hidden] { $\ir$ };
  \node (n) [above right=of s.east, yshift=0.2cm] { };
  \node (u) [right=of hr.east, hidden] { $u$ };
  \node [draw=black!50, fit={(state) (hr) (u) (n) }, label={environment}] {};

  \path (a) edge[->] (state);
  \path (d) edge[->] (r);
  \path (state) edge[->] (hr);
  \path (hr) edge[->] (r);
  \path (u) edge[->] (hr);
\end{tikzpicture}
\caption{Bayesian network}
\label{fig:bayesian-network}
\end{figure}
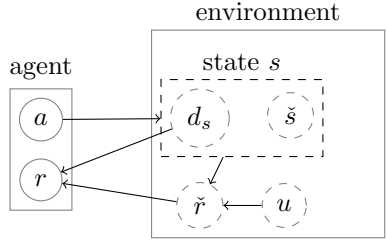

What if we instead start from a prior $C(u)$, and try to construct a
consistent distribution $B(s,r\mid a)$?
In addition to $C(u)$ we would need $B(s\mid a)$ and $B(r\mid s,\ir)$,
from which we may define
$B(r\mid s) = \sum_{u,\ir'}C(u)C(\ir\mid s,u)B(r\mid s,\ir)$.
The joint distribution
\begin{equation}\label{eq:bn}
  P(u,s,\ir, r\mid a) = C(u)B(s\mid a)C(\ir\mid s,u)B(\ir\mid s,r)
\end{equation}
is displayed as a Bayesian network in \cref{fig:bayesian-network}.

\begin{lemma}[Assumptions hold]\label{le:bn-as}
If $B(r\mid s,\ir')$ correctly specifies that
$d_s=\did \implies B(r\mid s,\ir')=\bbl  r = \ir'\bbr $,
then \cref{as:BC-consistency} holds.
\end{lemma}

\begin{proof}
Assume that $B(r\mid s,\ir')$ correctly specifies that
$d_s=\did \implies B(r\mid s,\ir')=\bbl  r = \ir'\bbr $.
Then \cref{as:BC-consistency} holds, since for any $r$ and any
$s$ with $d_s=\did$
\begin{align*}
B(r\mid s)
&= \sum_{u,\ir'}C(u)C(\ir'\mid s,u)B(r\mid s,\ir')\\
&= \sum_{u,\ir'}C(u)C(\ir'\mid s,u)\bbl r = \ir \bbr\\
&= \sum_{u}C(u)C(r\mid s,u) = C(r\mid s)
\end{align*}
the last equality by definition.
\end{proof}

The primary difficulty with this approach is how to correctly specify
$B(r\mid s,\ir)$. As we have discussed above, it is generally hard
to determine sensory modifications from an opaque state representation $s$.
Indeed, if one could design an agent around the distribution $P$
in \cref{eq:bn}, then one could let the agent optimise
$\tilde V(a) = \sum_{\ir}P(\ir\mid a)\ir$.
This would likely directly solve the wireheading problem,
since such an agent would strive to
optimise the inner reward $\ir=u(s)$, rather than the observed reward $r$.
We fear that it will too hard to properly define $B(r\mid s,\ir)$ in
most contexts, however.

\paragraph{Summary}
In this section we have discussed two approaches to designing agents
with consistent distributions $B$ and $C$.
While starting from $C$ gives the cleanest result in terms of satisfying
\cref{as:BC-consistency}, 
it also puts unrealistic demands on the designer
(how to define $B(r\mid s,\ir)$?).
Starting from $B$ seems more feasible: if $\U$ can be chosen finite,
it suffices to find a number of non-deluding states or actions,
by means of which $C(u)$ can be extracted from $B(r\mid s)$.
Open questions remain about this approach, however.
A third approach would be to find two
Solomonoff priors $B$ and $C$ that are sufficiently
consistent for the gist of \cref{th:niw} to go through.

\section{Direct Wireheading}
\label{sec:dw}

This section considers the argument made by \citet{Hibbard2012} that
the wireheading problem is avoided by utility agents.
We give a simple formal version of the argument,
and point out a shortcoming of the argument when applied to
general value learning agents.
In short, some utility functions may directly endorse self-delusion.
\Cref{sec:aisb} discusses a tentative approach for fixing this problem.

\subsection{Inner State Based Utility Functions}
For our argument it will be important to define
utility functions that only depend on the inner state $\is$ (\cref{def:is}) and
are independent of the self-delusion $d_s$.

\begin{definition}[isb utility function]
  \label{def:inner-state-utility}
  We write $s=\is d_s$, assuming that the states $s$ is fully described
  by the inner state $\is$ and the self-delusion $d_s$.
  We call a utility function $u$ \emph{inner state based (isb)} if
  $u(\is d_s)=u(\is\did)$ for any $\is $ and $d_s$,
  and write $u(s)=u(\is d_s)=u(\is)$.
  The utility function $u$ is
  \emph{$\epsilon$-approximately inner state based ($\epsilon$-isb)}
  with $\epsilon\geq 0$ if for all $\is$ and $d_s$
  $u(\is d_s)\epsapprox u(\is\did)$, where $\epsapprox$ means that
  the difference is at most $\epsilon$.
\end{definition}

\Citet{Hibbard2012} argued that wireheading is not a problem
if the agent tries to optimise a utility function $u$ that depends
on the (inner) state of the agent's world model.%
\footnote{
  Hibbard (private communication) argues that his
  \emph{model-based utility functions} \citep{Hibbard2012}
  are different in spirit to our isb utility functions.
  A similar non-wireheading argument seems to apply to both
  types of utility functions, however.
}
\Citeauthor{Hibbard2012} also argued that this is true even if
the world model is itself a mixture over different possible
world models. 
To distinguish Hibbard's non-wireheading result from our results in
\cref{sec:ind-wire},
we say that the agent \emph{directly wireheads} if it uses its
self-delusion ability to optimise a utility function directly,
rather than shift the evidence towards more easily satisfied utility
functions as in \cref{sec:niw,sec:examples}.

\begin{theorem}[No direct wireheading]
  \label{th:ndw}
  If $u$ is isb $u(s) = u(\is)$,
  then
  \begin{equation}\label{eq:ndw}
    V_u(a) = \sum_{\is}\Pr(\is\mid a)u(\is).
  \end{equation}
\end{theorem}

\begin{proof}
  The proof is immediate:
  \begin{align*}
    V_u(a)
    &= \sum_{s}B(s\mid a)u(s)\\
    &=  \sum_{s}B(\is d_s\mid a)u(\is)\\
    &=  \sum_{\is}B(\is\mid a)u(\is)
  \end{align*}
  where the last step marginalises $d_s$.
\end{proof}

That is, a $V_u$-based agent with an isb utility function $u$
will focus solely on optimising the inner
state $\is$, and have no incentive to self-delude.
In the chess example, the position of the chess board would be part of
the inner state $\is$.
If it was possible to determine the position of
the chess board from the agent's state representation,
one could design an agent with utility function $u(\is)=1$
for victory states $\is$,
and $u(\is')=-1$ for loss states $\is'$.
\Cref{th:ndw} shows that such an agent would have no incentive to
self-delude.

The isb assumption is necessary for \cref{th:ndw}.
Without this assumption it is possible to create self-deluding
utility agents, as illustrated by the following example.
The conclusion of the example
is consistent with other results on RL agents \citep{Ring2011},
and shows that the use of state-based utility functions is not
a guarantee against wireheading.
\begin{example}[Reward maximising utility agent]\label{ex:dw}
  A reward maximising utility agent is defined by the utility function
  $\ud(s)=d_s(u'(s))$, where $u'$ is some function generating the
  inner reward.
  The utility function $\ud$ strongly endorses self-delusion:
  The agent obtains maximal utility in
  states $s$ with delusion $d_s=\dor$ that clamps reward to 1,
  since $\ud(s)=\dor(u'(s))=1$. 
\end{example}

\subsection{CP-VRL with isb Utility Functions}

In value learning, the agent learns from experience which utility function
is the true one, starting from a prior $C(u)$ over a class $\U$ of
utility functions.
If all $u\in\U$ are isb, then any mixture $\u(s)=\sum_uC(u)u(s)$
will also be isb.
A CP-VRL agent that is built around a class $\U$ of isb utility functions
will therefore avoid the direct wireheading problem:

\begin{corollary}[No wireheading]\label{th:nw}
  Assume $\U$ contains only isb utility functions.
  Then, for the CP-VRL agent the value function reduces to
  \begin{equation}\label{eq:nw}
    V(a) = \sum_{u,\is}B(s\mid a)C(u)u(\is).
  \end{equation}
\end{corollary}

\begin{proof}
  Let $\u(s)=\sum_uC(u)u(s)$.
  Then $\u$ is isb, since $\u(s)=\sum_uC(u)u(s)=\sum_uC(u)u(\is)=\u(\is)$.
  Therefore, \cref{th:ndw,th:niw} give \cref{eq:nw}:
  \begin{align*}
    V(a)
    &\stackrel{\eqref{eq:niw}}{=}
      \sum_{u,\is}B(s\mid a)C(u)u(s)\\
    &= \sum_{\is}B(s\mid a)\u(s)\\
    &\stackrel{\eqref{eq:ndw}}{=}
      \sum_{\is}B(s\mid a)\u(\is).\qedhere
  \end{align*}
\end{proof}

However, if $\U$ includes functions such as $\ud$, then
direct wireheading may be a problem for value learning agents.
The problem is exacerbated by the fact that utility functions such as
$\ud$ will always be consistent with the observed reward $r$, self-delusion
or not.
On the other hand, functions of type $\ud$ may have sufficiently
small prior weight that direct wireheading will never induce
a sufficient incentive for the agent to wirehead.

\subsection{Approximately isb Utility Functions}
\label{sec:aisb}
This subsection briefly discusses a possibility for constructing a wide class
of  approximately isb ($\epsilon$-isb) utility functions.
The next result shows that if $u$ is $\epsilon$-isb,
then the incentive for the agent to self-delude is not strong.

\begin{theorem}[Almost no direct wireheading]
  \label{th:andw}
  If $u$ is $\epsilon$-isb $u(s) \epsapprox u(\is\did)$,
  then 
  \[
    V_u(a) \epsapprox \sum_{\is}\Pr(\is\mid a)u(\is\did).
  \]
\end{theorem}

\begin{proof}
  Assuming $u$ is $\epsilon$-isb, $V_u(a)$ is upper bounded by
  \begin{align*}
    V_u(a)
    &= \sum_{s}B(s\mid a)u(s)\\
    &\leq \sum_{s}B(\is d_s\mid a)(u(\is\did)+\epsilon)\\
    &=  \sum_{\is}B(\is\mid a)u(\is\did) + \epsilon
  \end{align*}
  and similarly lower bounded. From this follows that $V_u(a)$ deviates
  at most $\epsilon$ from $\sum_{\is}B(\is\mid a)u(\is\did)$.
\end{proof}

The following is one suggestion for constructing a wide class of
$\epsilon$-isb utility functions.
\begin{definition}[Convolutional utility functions]\label{def:conv}
  Assume that the states
  are represented as binary strings $\S=\{0,1\}^*$.
  For a string $s=s_1\dots s_{|s|}$, let $s_{m:n} = s_m\dots s_{n}$
  for $1\leq m\leq n\leq |s|$.
  Let $\tilde\U$ be the set of computable functions $\tilde u:\{0,1\}^k\to\R$,
  and let $\U^{{\rm cv}}$ be the set of \emph{$k$-convolutional utility functions}
  defined by
  \(
  \U^{{\rm cv}}:=\left\{ u(s)=\sum\nolimits_{i=1}^{|s|-k} \tilde u(s_{i:i+k}):
    u\in\tilde\U\right\}.
  \)
\end{definition}
Convolutional utility functions are suitable under the
following assumptions:
\begin{inparaenum}[(1)]
\item The agent's state representation has a
  \emph{similar topological structure}
  as the real world.
\item The principal cares approximately equally about all parts
  of the real world (for example, each
  place inhabiting a happy human contribute equally to the total
  utility of the state).
\item The state of the delusion box only affects a small part of
  the state representation (so the utility functions are approximately
  inner state based in the sense of \cref{th:andw}).
\end{inparaenum}
Further research should investigate the plausibility of these
assumptions, and whether constructions like \cref{def:conv}
are at all necessary.
Possibly, the class $\U$ of all computable utility functions
comes without substantial risks.

\section{Omitted Proofs}
\label{ap:omitted}

\begin{lemma}[U-VRL is RL]\label{le:u-vrl}
  $V(a)=\Vrl(a)$, so the U-VRL agent is equivalent to the RL agent.
\end{lemma}

\begin{proof}
  $V(a)$ may be written as
  \begin{equation}\label{eq:alt-V}
    V(a) = \sum_{s,r}B(s\mid a)B(r\mid s)\sum_uC(u\mid s,r)u(s).
  \end{equation}
  The sum over $u$ reduces to $r$, since
  \begin{align*}
    \sum_uC(u\mid r,s)u(s)
    &= \sum_u \frac{C(u)C(r\mid s,u)}{\sum_{u'}C(u')C(r\mid s,u')}u(s)\\
    &= \sum_u \frac{C(u)\bbl u(s)=r\bbr }{\sum_{u'}C(u')\bbl u'(s)=r\bbr }u(s)\\
    &= \sum_{u:u(s)=r} \frac{C(u)}{\sum_{u':u'(s)=r}C(u')}u(s)\\
    &= \sum_{u:u(s)=r} \frac{C(u)}{\sum_{u':u'(s)=r}C(u')}r = r
  \end{align*}
  Replacing the sum over $u$ with $r$ in \cref{eq:alt-V} gives $\Vrl$.
\end{proof}

\end{document}